\def\year{2022}\relax
\documentclass[letterpaper]{article} %
\usepackage{aaai22}  %
\usepackage{times}  %
\usepackage{helvet}  %
\usepackage{courier}  %
\usepackage[hyphens]{url}  %
\usepackage{graphicx} %
\usepackage{natbib}  %
\usepackage{caption} %
\DeclareCaptionStyle{ruled}{labelfont=normalfont,labelsep=colon,strut=off} %
\usepackage{algorithm}
\usepackage{algorithmic}

\usepackage{newfloat}
\usepackage{listings}
\floatstyle{ruled}
\newfloat{listing}{tb}{lst}{}
\floatname{listing}{Listing}
\usepackage{lineno}
\usepackage{amsmath, amsthm}
\usepackage{amssymb}
\usepackage{dsfont}

\DeclareMathOperator{\nextop}{\textsf{X}}
\DeclareMathOperator{\untilop}{\textsf{U}}
\DeclareMathOperator{\releaseop}{\textsf{R}}
\DeclareMathOperator{\eventuallyop}{\textsf{F}}
\DeclareMathOperator{\globallyop}{\textsf{G}}
\DeclareMathOperator{\weaknextop}{\textsf{N}}
\DeclareMathOperator{\weakuntilop}{\textsf{W}}

\newcommand{\var}{\mathrm{var}}
\newcommand{\LTL}{$\mathsf{LTL}$}
\newcommand{\LTLf}{$\mathsf{LTL}_f$}
\newcommand{\NeuralLTLf}{Neural$\mathsf{LTL}_f$}
\newtheorem{lemma}{Lemma}
\newtheorem{theorem}{Theorem}
\newtheorem{definition}{Definition}
\newcommand{\namecite}[1]{\citeauthor{#1}~\citeyearpar{#1}}

\title{Learning Finite Linear Temporal Logic Specifications \\ with a Specialized Neural Operator}
\author {
    Homer Walke \textsuperscript{\rm 1 \rm 2}
    Daniel Ritter \textsuperscript{\rm 1}
    Carl Trimbach \textsuperscript{\rm 1}
    Michael Littman \textsuperscript{\rm 1}
}
\affiliations {
    \textsuperscript{\rm 1} Brown University\\
    \textsuperscript{\rm 2} UC Berkeley\\
    homer\_walke@berkeley.edu, daniel\_ritter@brown.edu, carl\_trimbach@brown.edu, mlittman@cs.brown.edu
}

\begin{document}

\maketitle

\begin{abstract}
Finite linear temporal logic (\LTLf) is a powerful formal representation for modeling temporal sequences. 
We address the problem of learning a compact \LTLf\ formula from labeled traces of system behavior. 
We propose a novel neural network operator and evaluate the resulting architecture, \NeuralLTLf. 
Our approach includes a specialized recurrent filter, designed to subsume \LTLf\ temporal operators, to learn a highly accurate classifier for traces. Then, it discretizes the activations and extracts the truth table represented by the learned weights. This truth table is converted to symbolic form and returned as the learned formula. 
Experiments on randomly generated \LTLf\ formulas show \NeuralLTLf\ scales to larger formula sizes than existing approaches and maintains high accuracy even in the presence of noise.
\end{abstract}

\section{Introduction}
Recurrent neural networks (RNNs) have proven highly effective at learning classifiers for sequential data. Yet, RNNs typically employ a large number of parameters leading to a lack of interpretability in the decisions they make. Linear temporal logic (\LTL) and its finite variant (\LTLf) are alternative representations for classifying sequential data in a symbolic, human-understandable manner \cite{pnueli1977temporal,de2013linear}. However, learning \LTLf\ formulas has proven to be a difficult task. We propose \NeuralLTLf{}, a new technique for learning classifiers for temporal behavior that combines the ease of optimization of RNNs with the interpretability of \LTLf.

\LTLf\ learning techniques are central to specification mining, or the extraction of temporal logic formulas from the execution traces of programs for formal verification \cite{lemieux2015general}.
\LTLf\ is also applicable in learning from demonstrations. After a human teacher demonstrates the desired behavior to a learning agent, the agent produces an \LTLf\ formula summarizing the behavior~\cite{vazquez2018learning, kasenberg2017interpretable}, and the formula is used in place of a reward function in the context of reinforcement learning~\cite{littman2017environment, li2017reinforcement}. 

We examine the problem of producing a compact \LTLf\ formula that correctly classifies traces of system behavior given labeled examples.

\begin{definition}[\LTLf{} Learning Problem]
Given a set of finite-length positive traces, $\Pi_P$, and a set of finite-length negative traces, $\Pi_N$, produce a compact \LTLf{} formula satisfied by the positive traces and violated by the negative traces. 
\end{definition}

The \LTLf\ learning approach of \namecite{camacho2019learning} is most directly related to our work. They reduce the \LTLf\ learning problem to SAT, but their method does not scale well to larger formula sizes or trace sets. Unlike \NeuralLTLf{}, their method fails to find an appropriate formula when the data is noisy and not perfectly separable with a formula of the specified size.
\namecite{neider2018learning} combine SAT solving and decision trees to produce \LTL\ formulas in the presence of noisy data, but they face similar scaling issues.
\namecite{kim2019bayesian} use Bayesian inference to learn \LTL\ formulas for a limited set of 
\LTL\ templates. \namecite{Mao2021Temporal} also present a neural network architecture inspired by \LTL, but their method does not produce formulas.

Our contributions are:
\begin{enumerate}
    \item \NeuralLTLf{}, a method for producing \LTLf{} formulas that classify traces.
    \item Evaluation of \NeuralLTLf{} on synthetic data for qualitative \LTLf\ formulas.
    \item Comparison of \NeuralLTLf{} to SAT-based approaches.
\end{enumerate}

\section{Linear Temporal Logic}
Linear temporal logic (\LTL) is a formal language used to express temporal properties of sequential data. \LTL\ formulas consist of a set of propositions $p \in P$, standard logical operators, and temporal operators. Formulas are evaluated over traces, $\pi$, which are sequences of truth assignments to all the propositions in $P$. \LTL\ is defined on infinite-length traces. In this paper, we use a variant, \LTLf, defined on finite-length traces~\cite{de2013linear}. The notation $\pi, t \models \phi$ denotes that the formula $\phi$ holds at timestep $t$ in trace $\pi$ where $0 \leq t < T$, and $T$ is the trace length. When $\pi, 0 \models \phi$, we say the trace $\pi$ \emph{satisfies} the formula $\phi$. 

The minimal temporal operators are next ($\nextop$) and until ($\untilop$). Next, $\nextop \phi$, denotes that $\phi$ will hold in the following timestep, while until, $\phi \untilop \psi$, denotes that $\phi$ must hold until $\psi$ becomes true.
A number of temporal operators can be formed from these operators. The eventually ($\eventuallyop$) operator denotes that a variable holds at some timestep in the future: $\eventuallyop \phi \Longleftrightarrow \textit{true} \untilop \phi$. The globally ($\globallyop$) operator denotes that a variable holds at all subsequent timesteps: $\globallyop \phi \Longleftrightarrow \lnot \eventuallyop \lnot \phi$. \LTLf\ introduces an additional temporal operator, weak next ($\weaknextop$), to address behavior at the end of a trace. Weak next, $\weaknextop \phi$, denotes that $\phi$ must hold at the next time step \emph{or} the next time step does not exist. 
In this paper, we also use the weak until operator defined as $\phi \weakuntilop \psi \Longleftrightarrow (\phi \untilop \psi) \lor \globallyop \phi$. Weak until is similar to $\phi \untilop \psi$, except that $\psi$ does not need to occur. Finally, the fragment of \LTL\ consisting of only next and next-derived operators is called \emph{metric} \LTL, while the fragment consisting of only until and until-derived operators is called \emph{qualitative} \LTL. 

\section{\NeuralLTLf{}}

Inspired by the temporal operators, we present a novel network architecture for classifying traces. Layers in the network consist of multiple filters, similar to convolutional filters~\cite{fukushima1979neural}. The filters in each layer are ``soft'' versions of \LTLf\ operators. By stacking layers in the network, subsequent filters are applied to the results of previous filters, equivalent to the nesting of \LTLf\ operators in a formula. The entire network encodes a single \LTLf\ formula.

The first layer of our network takes as input traces from the positive and negative trace sets. The truth values of the propositions in the traces are interpreted as 1 and 0 for \textit{true} and \textit{false}, respectively. The filters in the first layer are applied to the trace to generate a sequence of activations in $[0, 1]$. These activations form a new trace, which becomes the input to the subsequent layer.
These intermediate traces represent the truth values of the soft \LTLf\ operators encoded by the filters. The activations of the final layer correspond to the network's prediction of the truth value for every timestep in the original trace. Corresponding to the semantics of \LTLf, we use the truth value of the first timestep in the network output as the predicted label of the trace. We compare the predicted and target labels of the trace to compute a loss that we minimize via gradient descent. Figure~\ref{fig:overview} depicts a \NeuralLTLf{} network.

\begin{figure}
    \centering
    \includegraphics[scale=0.4]{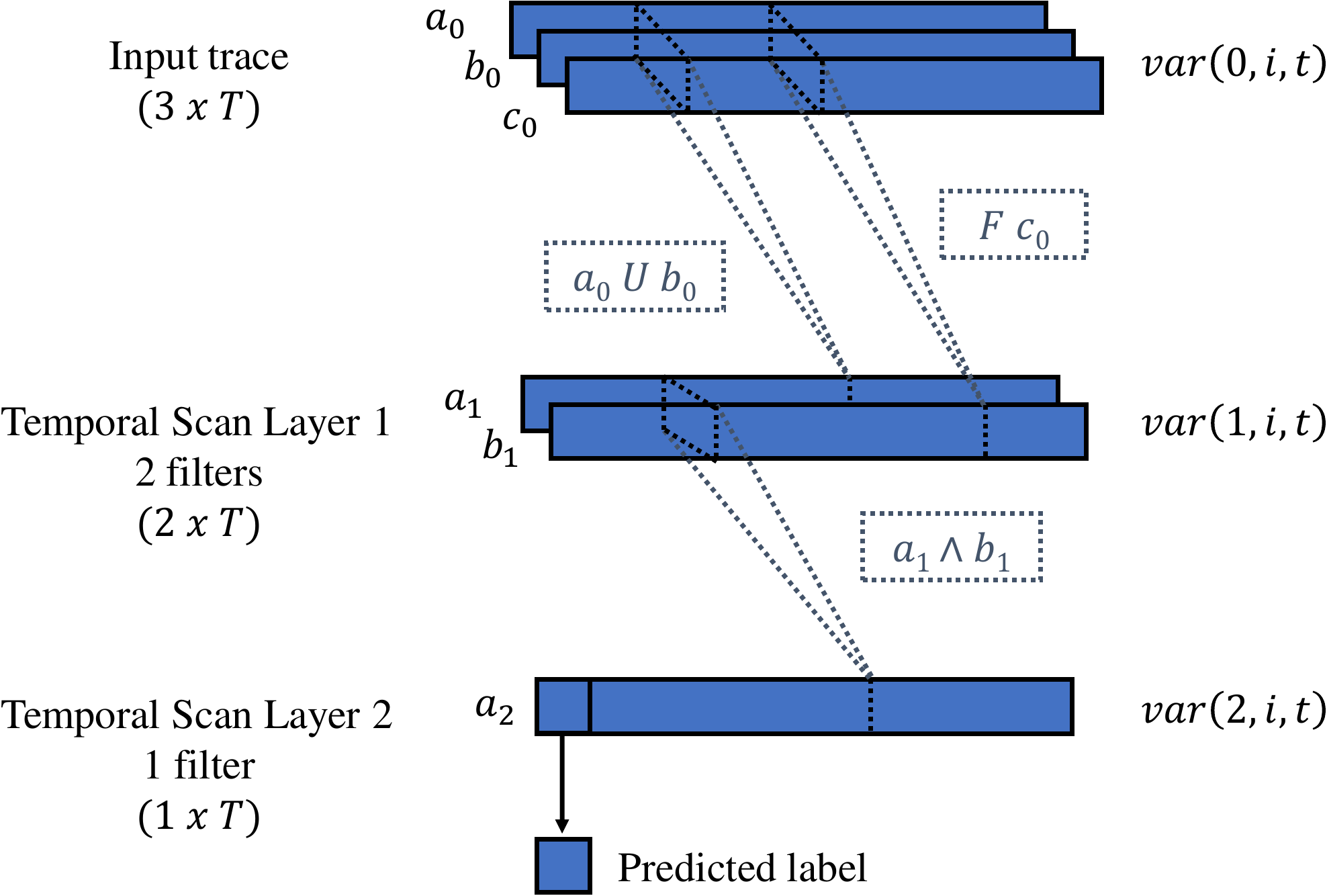}
    \caption{A \NeuralLTLf{} network that encodes the formula $(a \untilop{} b) \land \eventuallyop{} c$. Solid boxes are the input trace and output activations. Dashed lines represent the application of \NeuralLTLf filters. The formulas in the dashed boxes represent formula fragments learned by each filter.}
    \label{fig:overview}
\end{figure}

\subsection{Network Weights}
\label{sec:netweights}
Each layer, $l$, of the network consists of at least 1, but possibly multiple, filters indexed by $i$. These filters act on a sequence of truth values for variables indexed by $j$. The sequence of truth values might be the original trace, in the case of the first layer, or the output of the previous layer. We use $\var(l, i, t)$ to denote the activation of filter $i$ at timestep $t$ in layer $l$. The input trace is $\var(0, i, t)$.

A filter consists of a set of weights that allow for the expression of standard logical operators and temporal operators.
\begin{itemize}
    \item $W_P(l, i, j)$ is the propositional weight of filter $i$ in layer $l$ for variable $j$ and allows for the expression of standard logical operators. 
    \item $W_M(l, i, j)$ is the metric weight of filter $i$ in layer $l$ for variable $j$ and allows for the expression of metric temporal operators. 
    \item $W_Q(l, i)$ is the qualitative weight of filter $i$ in layer $l$ and allows for the expression of qualitative temporal operators. 
    \item $b(l, i)$ is the bias term for filter $i$ in layer $l$.
    \item $\var(l-1, j, T+1)$ and $\var(l, i, T+1)$ are base values.
\end{itemize}
Together, these weights define a linear classifier that gives the truth value of the soft \LTLf\ operator represented by the filter. The weights of one \NeuralLTLf{} filter can represent various \LTLf\ operators (Table~\ref{tab:weights}).

We apply a filter to a sequence using the following formula:
\begin{eqnarray}
\label{e:filter}
\var(l, i, t) =  \sigma \hspace{-8pt} &\biggl(\hspace{-8pt}&\sum_j W_P(l, i, j){\var}(l-1, j, t)\ \\
&& + \ \sum_j W_M(l, i, j){\var}(l-1, j, t+1)\ \nonumber \\
&& + \ \delta(W_Q(l, i)){\var}(l, i, t+1) + b(l, i)\biggr)\nonumber,
\end{eqnarray}
where $\delta$ is the function $\max(0, x)$, since we require that $W_Q$ is positive for formula extraction (Section~\ref{sec:conversion}). However, $\delta$ takes a slightly different form for training (Section~\ref{sec:implementation}). Similarly, $\sigma$ is the binary step function, $\mathds{1}_{[0, \infty)}$, for formula extraction and the sigmoid activation during training.

In words, applying a filter to a sequence is a recursive operation in the timestep $t$, corresponding to the recursive evaluation of an \LTLf\ operator on a trace. The recursion begins at timestep $T$ and the base case values, $\var(l-1, j, T+1)$ and $\var(l, i, T+1)$, are parameters learned along with the weights. Running backwards temporally, the output of a filter is computed as $\sigma$ applied to the sum of the propositional weights applied to the variables at the current timestep, the metric weights applied to the variables at the next timestep, and the qualitative weight applied to the output of the filter at the next timestep. 

The weights are trained via gradient descent to accurately capture the classification of the example traces.

\begin{table*}
\setlength{\tabcolsep}{3pt} %
\renewcommand{\arraystretch}{1.5} %
\begin{center}
\caption{Example weights for \LTLf\ operators: Assume one \NeuralLTLf{} filter, $i$, applied to two truth value sequences representing the \LTLf\ formulas $\phi$ and $\psi$. Standard logical operators are also easy to express using the $W_P$ weights. Weights that can take any value are marked with $-$. We abuse notation and use $W(l, i, \phi)$ to mean the weight applied to the truth value sequence representing $\phi$. }
\label{tab:weights}
\begin{tabular}{ | c | c | c | c | c | c | c | c | c | c |}
 \hline
 \LTLf\ Op. & $W_P(l, i, \phi)$ & $W_P(l, i, \psi)$ & $W_M(l, i, \phi)$ & $W_M(l, i, \psi)$ & $W_Q(l, i)$ & $b(l, i)$ & $\var(l-1, \phi, T+1)$ & $\var(l, i, T+1)$ \\
 \hline
  $\phi \untilop \psi$ & 1 & 2 & 0 & 0 & 1 & $-1.5$ & $-$ & 0 \\
 \hline
 $\phi \weakuntilop \psi$ & 1 & 2 & 0 & 0 & 1 & $-1.5$ & $-$ & 1 \\
 \hline
 $\nextop \phi$ & 0 & 0 & 1 & 0 & 0 & $-0.5$ & 0 & $-$\\
 \hline
 $\weaknextop \phi$ & 0 & 0 & 1 & 0 & 0 & $-0.5$ & 1 & $-$\\
 \hline
 $\eventuallyop \phi$ & 1 & 0 & 0 & 0 & 1 & $-0.5$ & $-$ & 0\\
 \hline
 $\globallyop \phi$ & 1 & 0 & 0 & 0 & 1 & $-1.5$ & $-$ & 1 \\
 \hline
\end{tabular}
\end{center}
\end{table*}

\subsection{Conversion from Network Weights to Formula}
\label{sec:conversion}

After a network has been trained, the learned weights can be interpreted as an \LTLf\ formula. In this interpretation, each \NeuralLTLf{} filter encodes an \LTLf\ expression that has the form $\phi \untilop \psi$ or $\phi \weakuntilop \psi$. Here, $\phi, \psi$ are Boolean expressions in full disjunctive normal form (fDNF) with additional literals for the next state of each proposition in each clause that are prepended by $\nextop$ or $\weaknextop$ operators. This space of augmented DNF expressions joined by $\untilop$ or $\weakuntilop$ will be referred to as \textit{temporal normal form} (TNF) expressions (see Section \ref{sec:example} for a TNF example). 

To facilitate the interpretation of \NeuralLTLf{} filters, we first define a concept we call a \textit{temporal truth table}. A temporal truth table is similar to a standard Boolean logic truth table, but is augmented with extra information specific to \NeuralLTLf{} filters. It encodes an \LTLf\ expression in temporal normal form
using the usual columns for each of the $n$ propositions $x_j$, and the output column $f$. In addition, for each $x_j$,
temporal truth tables have a corresponding column $m_j$, collectively called the metric bits. These bits encode the semantics of the metric operators $\nextop{}$ and $\weaknextop{}$ by representing the value of each proposition one timestep into the future. Additionally, there is a column $\tau$, the temporal bit, that represents the future truth value of the formula. It encodes the semantics of the qualitative operators $\untilop{}, \weakuntilop{}, \eventuallyop{}$ and $\globallyop{}$ since their truth values depend on future values in the trace. Separate from the columns, each temporal truth table has an additional $n+1$ bits of information that encode the filter's behavior at the end of the trace. The first of these bits, $\Omega$, is calculated by passing the filter's learned base case value $\var(l, i, T+1)$ through the binary step function. The other $n$ bits, $\omega_j$, are found by passing the learned base case values $\var(l-1, j, T+1)$ through the binary step function.

The rows for the temporal truth table are filled by applying the values for each of the bits in the truth table to the appropriate input of the filter (Equation \ref{e:filter}). 
The $x_j$ columns are multiplied by the $W_P$ weights, the $m_j$ columns are multiplied by the $W_M$ weights, and the $\tau$ column is multiplied by the $W_Q$ weight. The process of converting a filter to a temporal truth table discretizes the continuous operation of the filter. The conversion algorithm is outlined in Algorithm~\ref{alg:filter-table}. An example completed temporal truth table
is shown in Table~\ref{tab:truth}. 

\begin{table}[!t]
 \begin{center}
  \caption{Temporal truth table for the formula $x_1 \untilop x_2$. 
  The rows in which $f=0$ do not contribute and were omitted for space. The values of the literals from this table are applied to the weights of the learned filter, shown in the $\phi \untilop \psi$ row of Table~\ref{tab:weights} with $\phi=x_1$ and $\psi = x_2$. The resultant value determines $f$. Here, $\Omega$, $\omega_1$, and $\omega_2$ are calculated by the binary step function applied to $\var(l,i,T+1)$, $\var(l-1,x_1,T+1)$ and $\var(l-1, x_2,T+1)$, respectively.
  }
 \label{tab:truth}
  \begin{tabular}{|c|c|c|c|c|c|}
  \hline
  \multicolumn{2}{|c}{$\Omega=0$} & \multicolumn{2}{c}{$\omega_1=0$} & \multicolumn{2}{c|}{$\omega_2=0$}                         \\ \hline
  $x_1$ & $x_2$ & $m_1$ & $m_2$ & $\tau$ & $f$ \\ \hline
  0     & 1     & 0      & 0      & 0      & 1   \\
  0     & 1     & 0      & 0      & 1      & 1   \\
  0     & 1     & 0      & 1      & 0      & 1   \\
  0     & 1     & 0      & 1      & 1      & 1   \\
  0     & 1     & 1      & 0      & 0      & 1   \\
  0     & 1     & 1      & 0      & 1      & 1   \\
  0     & 1     & 1      & 1      & 0      & 1   \\
  0     & 1     & 1      & 1      & 1      & 1   \\
  1     & 0     & 0      & 0      & 1      & 1   \\
  1     & 0     & 0      & 1      & 1      & 1   \\
  1     & 0     & 1      & 0      & 1      & 1   \\
  1     & 0     & 1      & 1      & 1      & 1   \\
  1     & 1     & 0      & 0      & 0      & 1   \\
  1     & 1     & 0      & 0      & 1      & 1   \\
  1     & 1     & 0      & 1      & 0      & 1   \\
  1     & 1     & 0      & 1      & 1      & 1   \\
  1     & 1     & 1      & 0      & 0      & 1   \\
  1     & 1     & 1      & 0      & 1      & 1   \\
  1     & 1     & 1      & 1      & 0      & 1   \\
  1     & 1     & 1      & 1      & 1      & 1   \\ \hline
 \end{tabular}
 \end{center}
\end{table}

\begin{algorithm}
  \caption{Convert Filter to Temporal Truth Table}
  \label{alg:filter-table}
\begin{algorithmic}
  \STATE {\bfseries Input:} filter layer $l$, filter index $i$, trace length $T$, number of variables $n$
  \STATE $f \gets$ empty truth table
  \STATE $\Omega \gets \sigma(\var(l, i, T+1))$
  \FOR{$j \in \{1 \dots n\}$}
  \STATE $\omega_j \gets \sigma(\var(l, j, T+1))$
  \ENDFOR
  \FOR{$k \in \{0,1\}^{2n+1}$}
  \STATE $x_1,x_2,\ldots, x_n, m_1,m_2,\ldots,m_n,\tau\ \gets k$
  \STATE $f[k] \gets \sigma(\Sigma_j W_P(l, i, j)x_j + \Sigma_j W_M(l, i, j)m_j $
  \STATE \;\;\;\;\;\;\;\;\;\;\;\;\;\;$+ \delta(W_Q(l, i))\tau + b(l, i))$
  \ENDFOR
  \STATE \textbf{return} $f$, $\Omega$, $\omega$
\end{algorithmic}
\end{algorithm}
 
Interpreting the temporal truth table is straightforward because the table can be used to construct a formula in temporal normal form.
The operator is determined by $\Omega$, since whether an until operator is weak or strong is determined by the base case values:
$\untilop$ for $\Omega=0$ and $\weakuntilop$ for $\Omega=1$. The formula $\phi$ is created by taking the disjunction of the conjunction ~\cite{rautenberg2010concise} of all the proposition and metric bits when $f=1$ and $\tau=1$. Formula $\psi$ is created by taking the disjunction of the conjunction of the proposition and metric bits in the rows where $f=1$ and $\tau=0$. The metric bits are prepended with $\nextop$ when the corresponding $\omega=0$ and with $\weaknextop$ when the corresponding $\omega=1$, since the choice of $\nextop{}$ or $\weaknextop{}$ is determined by base case values.
The rows in which $f=0$ do not contribute to the expression's representation. Algorithm~\ref{alg:table-formula} outlines the procedure of converting a temporal truth table to a formula.

The conversion procedure is applied to each filter in a network and the resulting formulas are composed according to the structure of the network (see Figure \ref{fig:overview}). 
We prove the correctness of the conversion procedure in the Supplementary Material.

\begin{algorithm}[!t]
  \caption{Convert Temporal Truth Table to Formula}
  \label{alg:table-formula}
\begin{algorithmic}
  \STATE {\bfseries Input:} number of vars $n$, temporal truth table $f$, $\Omega$, $\omega$
  \STATE $\phi \gets$ False
  \STATE $\psi \gets$ False
  \FOR{$k \in \{0,1\}^{2n+1}$}
  \STATE $x_1,x_2,\ldots, x_n, m_1,m_2,\ldots,m_n,\tau\ \gets k$
  \IF{$f[k] = 1$}
  \STATE $c \gets$ True
  \FOR{$j \in \{1 \dots n\}$}
  \IF{$x_j = 1$}
  \STATE $b_j \gets x_j$
  \ELSE
  \STATE $b_j\gets \lnot x_j$
  \ENDIF
  \IF{$m_j = 1$}
  \STATE $d_j \gets x_j$
  \ELSE
  \STATE $d_j \gets \lnot x_j$
  \ENDIF
  \IF{$\omega_j = 1$}
  \STATE $c \gets c \land b_j \land \weaknextop{} d_j$
  \ELSE
  \STATE $c \gets c \land b_j \land \nextop{} d_j$
  \ENDIF
  \ENDFOR
  \IF{$\tau = 1$}
  \STATE $\phi \gets \phi \lor c$
  \ELSE
  \STATE $\psi \gets \psi \lor c$
  \ENDIF
  \ENDIF 
  \ENDFOR
  \IF{$\Omega = 1$}
  \STATE \textbf{return} $\phi \weakuntilop{} \psi$
  \ELSE
  \STATE \textbf{return} $\phi \untilop{} \psi$
  \ENDIF
\end{algorithmic}
\end{algorithm}

\subsection{Example of Conversion Procedure}
\label{sec:example}
As an example, we will carry out the conversion procedure for a filter, $A$, that has the weights listed in the first row of Table \ref{tab:weights}. Converting filter $A$ should result in the formula $x_1 \untilop x_2$. 

In the first phase, we create a temporal truth table, $f$, by evaluating filter $A$, using Equation \ref{e:filter}, for each setting of the propositional, metric, and temporal bits ($x_j, m_j$ and $\tau$).
Consider the bit setting, $k_1$, where $x_1 = 1, x_2 = 0, m_1 = 0, m_2 = 0$ and $\tau=1$. The setting $k_1$ corresponds to a trace that satisfies the formula $x_1 \untilop x_2$, because while the variable $x_2$ is \textit{false} ($x_2 = 0$), the variable $x_1$ is \textit{true} ($x_1 = 1$) and the formula is satisfied at a future state ($\tau = 1$). 
Evaluating filter $A$ with bit setting $k_1$, we have:
\[f[k_1] = \sigma((1 \cdot 1 + 0 \cdot 2) + (0\cdot0 + 0\cdot0) + \delta(1)\cdot1 - 1.5) = 1\]
as expected. Thus, the output column for row $k_1$ in the temporal truth table $f$ is set to 1. We also apply the binary step function to the learned base case values of the filter, $\var(l, i, T+1)$ and $\var(l-1, j, T+1)$, to produce $\Omega$ and $\omega_j$, respectively. We will assume $\var(l, i, T+1)$ and $\var(l-1, j, T+1)$ are 0, though they may take any value according to Table \ref{tab:weights}. So, we have:
\[\Omega = \sigma(\var(l, i, T+1)) = \sigma(0) = 0 \]
\[\omega_1 = \sigma(\var(l-1, 1, T+1)) = \sigma(0) = 0 \]
\[ \omega_2 = \sigma(\var(l-1, 2, T+1)) = \sigma(0) = 0.\]
Table \ref{tab:truth} shows the completely filled in temporal truth table. 

In the second phase, we convert the temporal truth table into an \LTLf\ formula. Since $\Omega = 0$, the formula will use an until operator and have the form $\phi \untilop \psi$. 
Then, each row in the temporal truth table with output 1 contributes a clause to $\phi$ or $\psi$. Once again, consider the row $k_1$ where $x_1 = 1, x_2 = 0, m_1 = 0, m_2 = 0$, and $\tau=1$. Row $k_1$ represents the clause:
\[x_1 \land \lnot x_2 \land (\nextop \lnot x_1)  \land (\nextop \lnot x_2).\]
Here, we use the strong next operator ($\nextop$) to represent the metric bits, rather than the weak next ($\weaknextop$) operator, because $\omega_1 = 0$ and $\omega_2 = 0$. We add this clause to $\phi$, rather than $\psi$, since $\tau = 1$:
\[\phi \gets \phi \lor (x_1 \land \lnot x_2 \land (\nextop \lnot x_1)  \land (\nextop \lnot x_2)).\]
Repeating the process for every row with output 1 in the table results in the complete sub-formulas $\phi$ and $\psi$. We now have a TNF formula, $\phi \untilop \psi$. Simplifying $\phi \untilop \psi$ using standard rewrite rules (discussed later) results in the formula $x_1 \untilop x_2$.

\subsection{Implementation Details}
\label{sec:implementation}
We implemented \NeuralLTLf{} in Tensorflow using binary cross-entropy loss optimized with Adam~\cite{Adam}. A formula that is satisfied by all the positive traces and violated by all the negative traces will have the minimum cross-entropy loss since the \NeuralLTLf{} network that encodes the formula will perfectly classify every trace. We employ several procedures that increase the accuracy and compactness of the formulas output by \NeuralLTLf{}. \\

\noindent \textbf{Logic Minimization} While every \NeuralLTLf{} filter can be converted into a TNF formula, the TNF formula is generally not human-readable due to its large size. Rather than returning a TNF formula directly, we use the Espresso logic-minimization algorithm to initially reduce the formula from the filter's temporal truth table into a more compact formula~\cite{rudell1986multiple}. Then, we use the Spot \LTLf\ library to further reduce the formula according to \LTLf\ simplification rules~\cite{duret2016spot}. Although Spot is designed for \LTL\ rather than \LTLf, we prove in the Supplementary Material that our use of Spot is valid for \LTLf. 
In our experiments, the average percent reduction in formula size by Espresso and Spot was 91\% and 51\%, respectively.
The initial TNF formulas had an average size of 3278367, the Espresso-reduced formulas had an average size of 650, and the Spot-reduced formulas had an average size of 17.  
\\

\noindent \textbf{Annealing and Random Restarts} There is inevitable information loss when converting the continuous network weights into a discrete temporal truth table. However, to encourage \NeuralLTLf{} to learn representations that maintain high accuracy when discretized, we linearly increase the steepness of the sigmoid activation, $\sigma$, as training progresses. Similarly, while we use the $\delta$ function to restrict $W_Q$ to positive values at test time, we relax this restriction during training. We define a ``leaky" $\delta$ with a small positive slope in the negative region. This negative-region slope is linearly reduced as training progresses. Specifically, given the definitions of $\sigma$ and $\delta$ parameterized by $\beta$ and $\alpha$: 
\[\sigma(x) = \frac{1}{1+e^{-\beta x}} \quad \mbox{and} \quad \delta(x) = \mathrm{max}(x, \alpha x). \]
We use annealing rates $\alpha_d$ and $\beta_d$, updating the values of $\alpha$ and $\beta$ at the end of each epoch by setting $\alpha = \alpha + \alpha_d$ and $\beta = \beta + \beta_d$.

Since formula extraction replaces the sigmoid activation with the binary step function and the leaky $\delta$ with the strict $\delta$, annealing these activations during training increases the likelihood that the extracted formulas will match the behavior of the optimized network. To further increase the chances of learning weights that discretize well, we also use random restarts. We train the network multiple times with different random weight initializations and use the trained network that has the highest accuracy after discretization. We attempted to employ $L_1$ regularization to the activations to further encourage better discretization, but found it made optimization too difficult for Adam. \\

\noindent \textbf{Multiple Networks} In principle, even a very large \NeuralLTLf{} network can produce a compact formula after simplification. However, since larger networks can represent larger formulas, they pose a greater risk of producing a formula that overfits the training data. To balance the goal of learning a compact but also highly accurate formula, we train multiple networks each with a different number of filters on a given set of data. We then choose the smallest formula of the set of formulas with the highest accuracy after extraction. In practice, the choice of network architecture serves as a way of incorporating domain specific knowledge of the formula structure one expects to learn, if such knowledge is available. Though, it is important to note our procedure does not rely on exactly matching the structure of the network with the structure of the formula.

\section{Experiments}

We evaluated our \LTLf\ formula learner, comparing it to approaches from the literature.

\begin{figure*}
    \centering
    \includegraphics[scale=0.32]{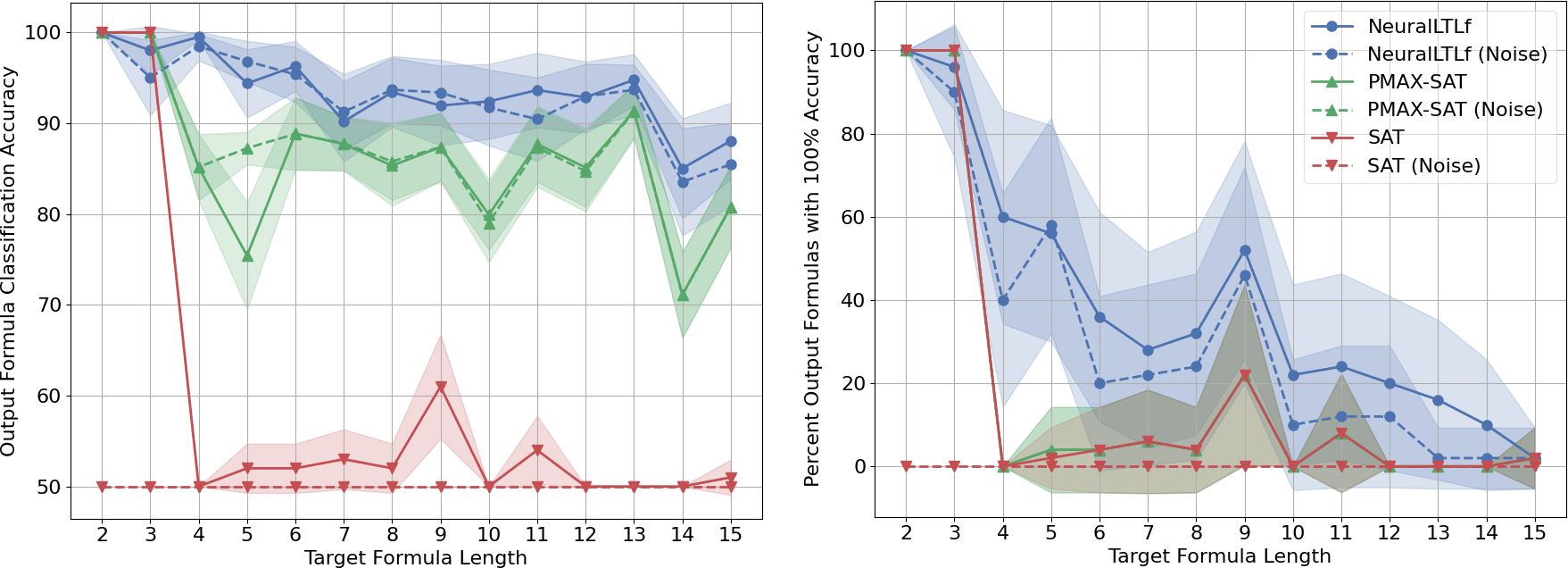}
    \caption{Comparison of the performance of \NeuralLTLf{}, and SAT and PMAX-SAT-based\textbf{} methods on the test set after training on data with and without noise. An accuracy of 50\% for the SAT method indicates the run timed out. 95\% confidence intervals shown.}
    \label{fig:compare}
\end{figure*}

\subsection{\NeuralLTLf{} vs.\ SAT}
To test the scalability of \NeuralLTLf{} with respect to formula size, we evaluated its performance on data from random formulas. Then, using the same data, we swapped 1\% of the labels to additionally test \NeuralLTLf{}'s ability to handle noise. For both experiments, we compared \NeuralLTLf{} with the SAT-based approach by \namecite{camacho2019learning} since their method does not make use of \LTLf\ templates to restrict the space of learnable formulas, like \namecite{kim2019bayesian}, and works out-of-the-box with \LTLf\ rather than \LTL, unlike \namecite{neider2018learning}.
We use their SAT encoding in conjunction with the associated learning algorithm. The algorithm iteratively increases the maximum allowed formula size and reruns the SAT solver until a formula is found. This process guarantees the output formula is optimally compact. 

To increase robustness to noisy labels, we also devised a novel variant of the SAT approach. In the partial maximum satisfiability (PMAX-SAT) problem~\cite{cha1997local},
rather than simply finding a satisfying truth assignment for a Boolean formula, the goal is to satisfy the \emph{maximum} number of a designated set of ``soft'' clauses, while satisfying all of the remaining ``hard'' clauses. Our PMAX-SAT variant uses the same SAT encoding from~\namecite{camacho2019learning}, but designates the clauses enforcing trace satisfaction as soft clauses. Thus, satisfying the maximum number of soft constraints in the PMAX-SAT problem corresponds to producing a formula satisfied by the maximum number of traces. Given the PMAX-SAT problem encoding, we execute a PMAX-SAT solver to learn a formula from the trace data. With the PMAX-SAT variant, we wanted to test whether a modified SAT-based approach could handle noise without prohibitively increasing runtime. \\

\noindent \textbf{Data} First, we generated random qualitative \LTLf\ formulas with $|P| = 3$ by uniform sampling of the \LTLf\ grammar. We generated 50 of each length ranging from 2 to 15 (or as many as possible if the number of unique formulas of a given size was less than 50). The length of an \LTLf\ formula is the sum of the number of temporal operators, binary logical operators, and propositions in the formula.
We threw out formulas that did not include a temporal operator, meaning there were no formulas of size 1.
We converted the formulas into negative normal form following the precedent set by \namecite{camacho2019learning}. Then, we adopted an approach from \namecite{camacho2019learning} and generated a \textit{characteristic sample} of traces for each formula's corresponding minimal deterministic finite-state automaton (DFA) \cite{parekh2001learning}. A set of labeled traces is considered \emph{characteristic} if the set uniquely defines a minimal DFA over a fixed number of states, $N$. Including a characteristic sample as part of the training data 
discouraged each method from oversimplifying the formula.
We mixed the characteristic sample with uniformly sampled random traces such that $|\Pi_P| = |\Pi_N| = 500$ for all formulas. We explored the alternative of using solely random traces, but found that the algorithms reliably found shortcut solutions that did not capture the true target formula. The mix of the characteristic sample and random traces produced much more reliable results. Lastly, the labels of 1\% of the total 1000 traces were inverted to produce a noisy dataset. We resampled the random traces for each formula to create the test data.

The last timestep of each trace in the characteristic sample was repeated such that all traces had length 15. Repeating the last timestep of a trace is guaranteed not to change its truth values with respect to a qualitative formula, as qualitative formulas define stutter-invariant languages \cite{peled1997stutter}. 
However, padding may change the truth values of traces for metric formulas. Because of the complexities of batch training on variable length data, we chose to only use qualitative formulas in our experiments. Accordingly, both \NeuralLTLf{} and the SAT-based methods were modified to only produce qualitative formulas. 

Since our intention was to test the scaling capabilities of each method, these datasets were produced with larger formulas (max size 15 vs 11) than those tested by \namecite{camacho2019learning}. \\

\noindent \textbf{Procedure} Each method was given a maximum runtime of 5 minutes per formula. \NeuralLTLf{} was allowed 3 network architectures each with 1 random restart. Table \ref{tab:networks} displays the chosen architectures. The batch size was set at 100 and the learning rate at 0.005. Each network was run for 3000 epochs or until accuracy after discretization reached 100\%. The sigmoid and ReLU activations were linearly annealed with rates $\beta_d = 0.01$ and $\alpha_d = {-7}\mathrm{e}{-5}$ respectively. All hyperparameters for \NeuralLTLf{}, including network architectures, were chosen via experimentation on held out random formulas. 
The SAT-based methods were run with solvers from Z3~\cite{de2008z3}. If a SAT-based method failed to produce any formula in the alloted time, we defaulted to the formula $true$ (which gives 50\% accuracy). Experiments were conducted on Debian machines with Intel Core i5-4690 CPUs at 3.5 GHz and 8 GB of RAM.

\begin{table}
    \centering
    \caption{The 3 \NeuralLTLf{} network architectures used in the experiment on random formulas. The filter assignments denote the number of filters in each layer with the input layer on the left and the output layer on the right.}
    \label{tab:networks}
    \begin{tabular}{c|c|c}
        Network & Layers & Filter Assignment \\
         1 & 1 & 1 \\
         2 & 2 & $3 \rightarrow 1$ \\
         3 & 3 & $5 \rightarrow 5 \rightarrow 1$
    \end{tabular}
\end{table}

\subsection{Results}
We used accuracy, defined as the percentage of correctly classified traces, as a performance metric to compare approaches. Figure~\ref{fig:compare} shows the performance of \NeuralLTLf{}, SAT, and PMAX-SAT on the test datasets after training on the original and noisy datasets. In both settings, \NeuralLTLf{} consistently produced formulas with high accuracy over all target formula lengths. The standard SAT approach by \namecite{camacho2019learning} began to time out on most formulas past a target formula length of 3. While \namecite{camacho2019learning} test the scalability of their approach using an active learning setup with no more than 40 traces per formula, we used passive learning and 1000 traces per formula which caused the method to time out on much smaller formulas.
Additionally, the SAT approach timed out on all formulas in the noisy setting. However, our PMAX-SAT variant performed significantly better than the standard SAT approach. In both settings and over all target formula lengths, our PMAX-SAT variant produced formulas with only slightly worse accuracy than \NeuralLTLf{}.

Further investigation into the formulas produced by \NeuralLTLf{} and PMAX-SAT revealed that \NeuralLTLf{} produced larger formulas on average (Figure~\ref{fig:length}). PMAX-SAT was unable to produce formulas larger than size 3 in the allotted time. Since in a majority of cases, a size 3 formula was smaller than the target formula, PMAX-SAT sacrificed accuracy for size. With larger formulas, \NeuralLTLf{} was able to fit to more patterns in the data and achieve higher accuracy than the SAT-based methods. However, unlike the SAT-based methods, \NeuralLTLf{} is not guaranteed to produce an optimally compact formula, and in some instances \NeuralLTLf{} produced very large, unintelligible formulas. For instance on data for a target formula $a \lor \globallyop{} \lnot c$, one \NeuralLTLf{} network learned the formula $(a \lor (((a \land \lnot c) \lor (\lnot b \land \lnot c)) \land \globallyop{} \lnot c)) \releaseop{} (a \lor ((a \lor (b \land \lnot c)) \land \globallyop{} \lnot c) \lor (((a \land \lnot c) \lor (\lnot b \land \lnot c)) \land \globallyop{} \lnot c))$. This formula perfectly classified the data, but clearly its size is undesirable. We set a maximum formula-size threshold of 25 as an informal notion of readability. When selecting an output formula from those produced by the 3 networks we trained for each target formula (Table~\ref{tab:networks}), we ignored those larger than 25.

\begin{figure}
    \centering
     \includegraphics[scale=0.32]{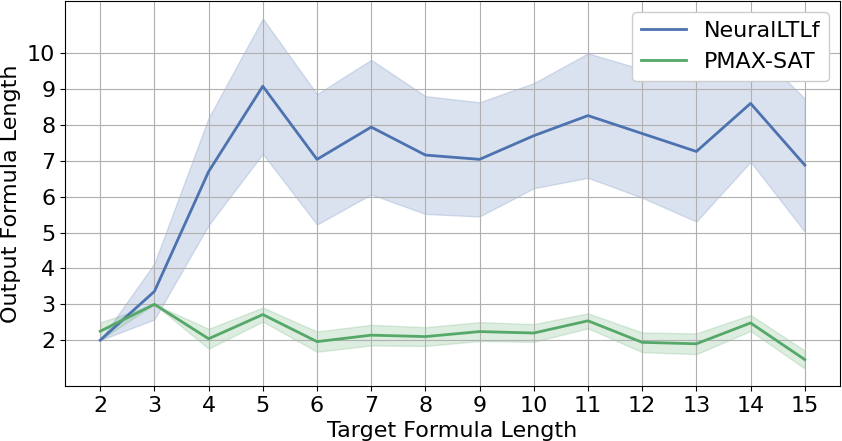}
    \caption{The length of formulas produced by \NeuralLTLf{} and PMAX-SAT on the non-noisy data with 95\% confidence intervals shown.}
    \label{fig:length}
\end{figure}

We also calculated the percentage of output formulas with 100\% classification accuracy for each method. The percentage of formulas with perfect accuracy produced by PMAX-SAT closely tracked the percentage of formulas with perfect accuracy produced by SAT, indicating that PMAX-SAT was able to find a perfect formula in nearly all cases for which SAT did not timeout. Notably, \NeuralLTLf{} was able to find significantly more formulas with 100\% accuracy at greater target formula lengths than the others.

As an example of the qualities of the formulas produced by the different methods, consider the target formula $b \lor \globallyop{} \lnot a \lor (b \releaseop{} a)$ of size 8. When given the data for this formula, \NeuralLTLf{} produced the exact target formula. The SAT method timed out and the PMAX-SAT method produced the formula $b$, which gave 91\% accuracy. While $b$ captures part of the target formula and classifies a majority of the traces correctly, much of the original formula's nuance is lost. Since the SAT-based methods could only produce formulas up to size 3 in the allotted time, the formulas produced by these methods often lacked relevant components. \NeuralLTLf{}'s ability to produce larger formulas in a shorter amount of time enabled it to find more complete formulas that better fit the data.

\section{Discussion}
We presented \NeuralLTLf{}, a neural network solution to the \LTLf\ learning problem, and evaluated its ability to scale to larger formulas as well as its robustness to noise. When tested on data sampled from random formulas, we found that \NeuralLTLf{} is capable of producing more accurate formulas on more complex tasks than the SAT-based approaches. When tested on a noisy version of the same data, we found that \NeuralLTLf{}’s performance was minimally affected.

However, there are a number of points at which \NeuralLTLf{} may fail to produce both a highly accurate and interpretable formula. During formula extraction, information can be lost when the activations of the network are discretized. The extracted formulas were on average 1\% less accurate than the trained \NeuralLTLf{} networks. We sought to increase the probability the networks would learn representations that discretize well by annealing the activation functions and using random restarts.

Additionally, \NeuralLTLf{} filters are highly expressive. A network architecture consisting of a small set of filters can represent a multitude of \LTLf\ formulas. Larger network architectures tended to learn formulas that were too large for human readability. Comparing multiple network architectures on any given dataset helped to alleviate this issue. 

Lastly, constructing and minimizing the temporal truth table in the formula-extraction step can require significant computational effort when the number of filters or propositions is large. The number of rows in the truth table is exponential in these values. Nevertheless, our experiments indicate that \NeuralLTLf{} does not suffer from scaling issues to the same degree as existing approaches. 

Restricting the expressiveness of \NeuralLTLf{} filters would help to address these issues. A smaller space of expressions would limit information loss during discretization, reduce the probability of unintelligible formulas, and allow for a more efficient formula-extraction procedure. We leave these topics for further research.

\bibliography{ltl}

\newpage

\section{Correctness of Conversion Procedure}

Here we prove the correctness of the procedure that converts the learned filter weights into an temporal truth table and then an \LTLf\ formula.

Because temporal truth tables represent TNF expressions, there are some settings of the table that result in logically impossible expressions, and are therefore invalid. Specifically, for any given setting of the $x_j$, $m_j$ bits, it cannot be the case that the row with $\tau=0$ has $f=1$ and the corresponding row with $\tau=1$ has $f=0$. If this were the case, that would mean that some clause of the temporal expression appears in $\psi$, but explicitly does not appear in $\phi$. This situation cannot occur because a clause's existence in $\psi$ guarantees that it is implicitly in $\phi$, by definition of the until operation. Tables that have such a property that create logically impossible expressions will be referred to as \emph{invalid}. By design, \NeuralLTLf{} filters create only valid truth tables when trained.

\begin{lemma}{Any \NeuralLTLf{} filter will produce a valid temporal truth table.}
\label{lem:valid}
\end{lemma}
\begin{proof}
An invalid truth table results when some setting of the $x_j$ and $m_j$ bits produces an output of $f=0$ when $\tau=1$ and $f=1$ when $\tau=0$. We assume for contradiction that we have an invalid truth table.

Using Equation~1,
$\tau$ from the truth table is represented by $\var(l,i,t+1)$ and the filter activation, $f$, is $\var(l,i,t)$. Consider some filter $i$ applied to identical settings of the propositional variables, but when $\var(l, i, t+1)=1$ then $\var(l,i,t)=0$, and when $\var(l,i,t'+1)=0$ then $\var(l,i,t')=1$. This situation is precisely what would cause the filter to produce an invalid table. Note that
\begin{align*}
  \var(l,i,t) &< \var(l,i,t') \\
  \delta(W_Q(i)){\var}(l,i,t+1) &< \delta(W_Q(i)){\var}(l,i,t'+1) \\
  \delta(W_Q(i))  &< 0.
\end{align*}
The second line is obtained by substitution from Equation~1.
The derivation shows that, for a filter to create an invalid table, $\delta(W_Q(i))<0$. However, $\delta = \text{max}(0, x)$, so $\delta(W_Q(i))$ is non-negative for any \NeuralLTLf{} filter. Therefore, all \NeuralLTLf{} filters produce valid truth tables.
\end{proof}

Beyond the filters encoding only valid truth tables, it is important that the method for interpreting those tables from and into \LTLf\ formulas is correct. Specifically, it should be the case that \NeuralLTLf{} filters can successfully be interpreted into \LTLf\ expressions. We note here that a successful interpretation is one that results in a valid expression that approximates, but need not exactly match the behavior of, the \NeuralLTLf{} filter---that happens since the interpretation uses a binary step function to discretize the operation of the filter. We also require that a successful interpretation create a temporal truth table that is \emph{\LTLf-expression preserving}. That is, any temporal truth table created from a given \LTLf\ expression will result in an equivalent \LTLf\ expression.

\begin{theorem}
Given a learned \NeuralLTLf{} filter, the process of interpreting its weights into \LTLf\ expressions is correct---\NeuralLTLf{} filters encode valid temporal truth tables that are \LTLf-expression preserving. That is, given an \LTLf\ expression $g$ and its temporal truth table $T$, one can create a formula $h$ from $T$ via Algorithm~2. Then, $g=h$ and $T$ is a valid temporal truth table.
\end{theorem}
\begin{proof}
By Lemma~\ref{lem:valid}, any temporal truth tables created by a \NeuralLTLf{} filter are valid. To prove the soundness of our interpretation method, we must show that any valid temporal truth table is \LTLf-expression preserving. 

Consider an arbitrary \LTLf\ expression $g$, in TNF. Evaluating this expression for every assignment of the variables in the temporal truth table will allow us to construct a valid temporal truth table.  Algorithm~2 creates a TNF expression from the table, $h$. Assume for contradiction that $g$ and $h$ differ in some way. For $h$ to differ from $g$, it must be missing a clause, have an additional clause in $\phi$ or $\psi$, or have a different operator than $g$.

If $h$ is missing a clause that was in $g$, that implies that the value of the temporal truth table for that clause was $0$. However, if that clause was in $g$, then its value in the table would have been $1$---a contradiction.

If $h$ has an extra clause that $g$ does not have, the value of that clause in the table was $1$. However, if that clause was not present in $g$, then the corresponding value of the table for that row would be $0$---a contradiction.

For $h$ to have a $\untilop$ where $g$ has a $\weakuntilop$ or an $\nextop$ where $g$ has a $\weaknextop$ or vice versa, it would need to have a 0 where $g$ has a 1 or vice versa in the $n+1$ extra bits of the truth table. That would contradict that the truth table was computed from $g$.

Any difference in $g$ and $h$ results in a contradiction in the structure of the temporal truth table, therefore $g$ and $h$ are identical. This argument shows that our method of interpreting truth tables is \LTLf-expression preserving. Since \NeuralLTLf{} filters encode only valid tables (Lemma~\ref{lem:valid}), and our method for interpreting valid tables is sound, we can interpret any \NeuralLTLf{} filter as a valid \LTLf\ expression.
\end{proof}

\section{Using Spot for \LTLf\ Simplification}

Spot is a library we use to simplify formulas extracted from \NeuralLTLf\ networks \cite{duret2016spot}. However, Spot is designed for \LTL\ not \LTLf\ so its use requires justification. We show any simplication rule that is valid for \textit{qualitative} \LTL\ is also valid for \textit{qualitative} \LTLf. Since we only test qualitative formulas in our experiments, using Spot for simplification does not introduce any invalid simplifications.

Given a trace $\pi$, $\pi_t$ is the truth assignment at timestep $t$ and $\overline{\pi_t}$ denotes that a timestep is repeated infinitely. We first prove the following useful lemma.

\begin{lemma}
Take a finite trace $\pi^f = \pi_0...\pi_n$ and repeat the last timestep to create an infinite trace, $\pi^i = \pi_0...\pi_{n-1}\overline{\pi_n}$. Then given a qualitative formula $\phi$, $\pi^f$ satisfies $\phi$ interpreted as \LTLf\ if and only if $\pi^i$ satisfies $\phi$ interpreted as \LTL. That is, $\pi_0...\pi_n \models \phi \iff \pi_0...\pi_{n-1}\overline{\pi_n} \models \phi$.
\end{lemma}
\begin{proof}
Qualitative \LTL\ and \LTLf\ formulas are stutter-invariant, so repeating timesteps or removing timesteps does not change the truth value of a trace \cite{peled1997stutter}.
\end{proof}

\begin{theorem}
All qualitative \LTL\ rewritings are also valid \LTLf\ rewritings.
\end{theorem}

\begin{proof}
Consider the qualitative \LTL\ rewriting $\phi \equiv \psi$. That is, for infinite traces $\pi^i \models \phi \iff \pi^i \models \psi.$ We want to show for \LTLf\ on finite traces $\pi^f \models \phi \iff \pi^f \models \psi.$ 

Take a finite trace $\pi^f = \pi_0...\pi_n$. By Lemma 1, if $\pi_0...\pi_n \models \phi$, then the infinite trace $\pi_0...\pi_{n-1}\overline{\pi_n} \models \phi$. Then because $\phi \equiv \psi$, we have $\pi_0...\pi_{n-1}\overline{\pi_n} \models \psi$. Again by Lemma 1, the finite trace $\pi_0...\pi_n \models \psi$. Thus for every finite trace $\pi^f \models \phi \implies \pi^f \models \psi$. The same argument applies to show $\pi^f \models \psi \implies \pi^f \models \phi$. So $\pi^f \models \psi \iff \pi^f \models \phi$ and the rewriting $\phi \equiv \psi$ is valid for \LTLf. 
\end{proof}

Because all qualitative \LTL\ rewriting are also valid \LTLf\ rewritings, our use of Spot is valid.

\section{Precision and Recall}

While we use accuracy as our primary comparison metric, we show precision and recall statistics for our experiments in Figure~\ref{fig:precision} and Figure~\ref{fig:recall}. These metrics follow similar trends as the accuracy metric. Though, the recall for the SAT method is always 1 since the method always produces formulas that perfectly classify the data, expect when it times out and defaults to $true$. However, $true$ has no false negatives. 

\begin{figure}
    \centering
    \includegraphics[width=\columnwidth]{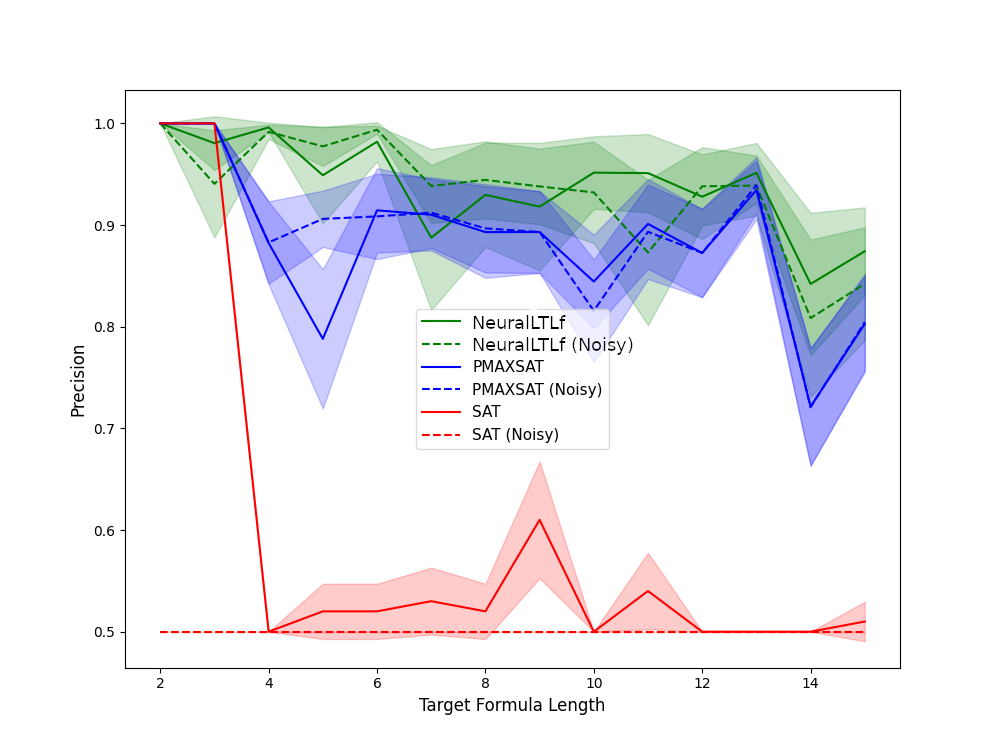}
    \caption{Precision statistics for both the original and noisy synthetic data. 95\% confidence intervals shown.}
    \label{fig:precision}
\end{figure}

\begin{figure}
    \centering
    \includegraphics[width=\columnwidth]{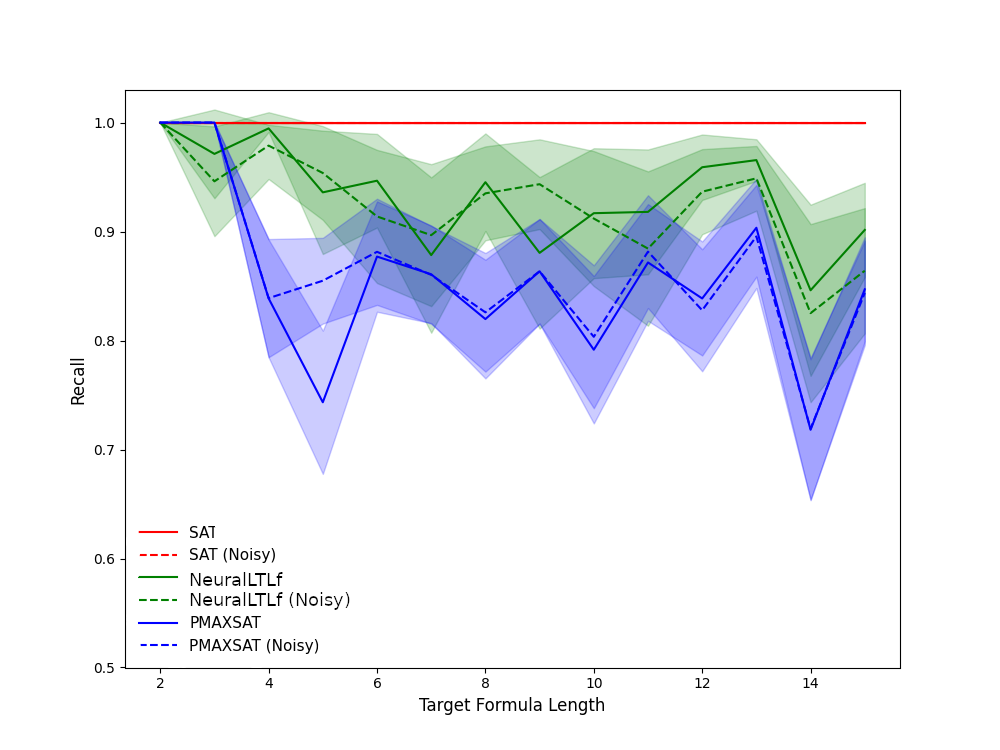}
    \caption{Recall statistics for both the original and noisy synthetic data. 95\% confidence intervals shown.}
    \label{fig:recall}
\end{figure}

\end{document}